\documentclass[letterpaper, 10 pt, conference]{ieeeconf}  

\IEEEoverridecommandlockouts                              

\usepackage{graphics} 
\usepackage{epsfig} 
\usepackage{mathptmx} 
\usepackage{times} 
\usepackage{amsmath} 
\usepackage{amssymb}  
\usepackage{leftidx}
\usepackage{mathtools}
\usepackage{tensor}
\usepackage{theorem}
\usepackage[english, ruled, vlined]{algorithm2e}
\usepackage{algorithmic}
\usepackage{mathtools}
\usepackage{thmtools}
\usepackage{thm-restate}
\usepackage{cleveref}

\declaretheorem[name=Remark]{remark}

\declaretheorem[name=Claim]{claim}
\declaretheorem[name=Note]{note}

\usepackage{authblk}
\usepackage{cite}
\usepackage{eucal}
\usepackage{subcaption}

\let\labelindent\relax
\usepackage[inline]{enumitem}

\usepackage{xcolor}
\begin{document}
\title{Perception-aware time optimal path parameterization for quadrotors}
\author{Igor~Spasojevic*,
        Varun~Murali*,
        and~Sertac~Karaman
\thanks{*I. Spasojevic and V. Murali contributed to this work equally.}%
\thanks{All authors are with the Department of Aeronautics and Astronautics and the Laboratory for Information and Decision Systems, Massachusetts Institute of Technology (MIT), Cambridge, MA 02139, USA.	
{\tt\small $\{$igorspas, mvarun, sertac$\}$@mit.edu}}
}
\maketitle
\thispagestyle{empty}
\pagestyle{empty}

\begin{abstract}
The increasing popularity of quadrotors has given rise to a class of predominantly vision-driven vehicles. 
This paper addresses the problem of perception-aware time optimal path parametrization for quadrotors. 
Although many different choices of perceptual modalities are available, the low weight and power budgets of quadrotor systems makes a camera ideal for on-board navigation and estimation algorithms.
However, this does come with a set of challenges. 
The limited field of view of the camera can restrict the visibility of salient regions in the environment, which dictates the necessity to consider perception and planning jointly.
The main contribution of this paper is an \emph{efficient time optimal path parametrization algorithm for quadrotors with limited field of view constraints.}
We show in a simulation study that a state-of-the-art controller can track planned trajectories, and we validate the proposed algorithm on a quadrotor platform in experiments.
\end{abstract}

\section{Introduction}

Rapid progress in autonomous aerial vehicles has secured their place in diverse applications, ranging from automated delivery, environmental monitoring, search and rescue to aerial cinematography.
A substantial amount of research is being devoted towards enabling them to perform tasks more accurately, more robustly, at operational speeds without any human oversight.
A key component to realizing such goals involves developing systems capable of simultaneously executing specified tasks while ensuring accuracy of future state estimates, obtained using solely percepts from on-board sensors.
As an example, consider the task of the AlphaPilot challenge \cite{AlphaPilot} (shown in Fig. \ref{fig:AlphaPilot}), which involved minimizing the time required by an autonomous quadcopter to complete a given racecourse. 

In this paper, we deal with light weight quadrotor aerial vehicles equipped with an inertial measurement unit (IMU) and a camera. 
An IMU supplies measurements of acceleration and angular velocity, which yield accurate state estimates over short time scales. 
However, over long horizons, such estimates accrue large drift which can be significantly reduced by using the differential bearing to landmarks, visually salient regions in the environment, captured by the on-board camera.     

A challenge of taking advantage of an on-board camera stems from its limited field of view and the underactuated nature of the quadrotor dynamics. Indeed, the four motors of the quadrotor can only exert thrust parallel to its body $z$ axis, implying its translational motion is intimately coupled with its orientation. In particular, a change in acceleration generally necessitates a change in attitude, and as a result, a change in the set of landmarks within field of view of the camera. The latter requirement imposes additional constraints on the trajectory planning task, which is already challenging due to the dimensionality of the state space of the quadrotor. 

A pivotal work \cite{MellingerKumar} demonstrated differential flatness \cite{Fliess} of quadrotor dynamics, allowing efficient trajectory planning in a lower dimensional space. \cite{MellingerKumar} leverage this insight to design trajectories which visit a set of waypoints at given times by solving a constrained quadratic program, while \cite{RichterBryRoy} enhance the latter approach by considering an alternate parametrization which results in a more numerically robust algorithm. \cite{HehnDAndrea} sample a suitable family of trajectories, incorporating \cite{MellingerKumar} to swiftly eliminate candidates which violate dynamic constraints. \cite{LiuAtanasovKumar} perform search in a subspace of flat outputs to generate collision-free trajectories which minimize a linear combination of execution time and the magnitude of jerk.

\begin{figure}[!t]
\centering
\includegraphics[width=8cm]{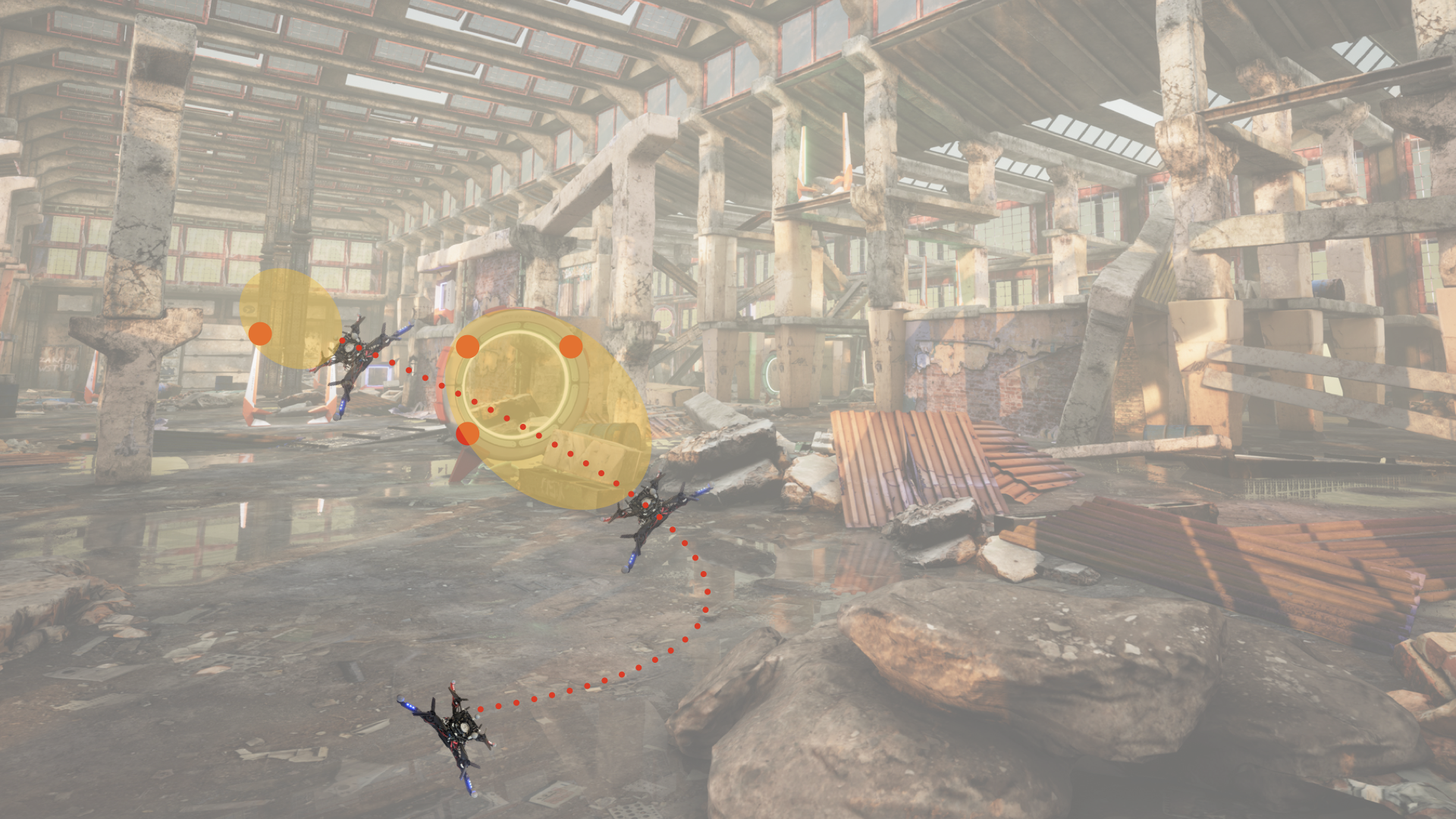}
\caption{In the figure above, the desired translational path is shown using the dotted line and the landmarks are depicted using red dots. The field of view is shown by the yellow ovals.}
\label{fig:AlphaPilot}
\end{figure}

However, none of the former approaches planned trajectories with perception constraints taken into account. \cite{PeninChaumette} plan minimum time trajectories in differentially flat space while taking into account the full dynamics of the quadrotor as well as requiring a particular target be maintained in field of view of the on-board camera. The output of their algorithm is a locally optimal dynamically feasible trajectory. \cite{SheckellsKobilarov} plan trajectories which guide the quadrotor towards a fixed final position while minimizing the integral of the magnitude of the reprojection error of a desired set of landmarks during the execution of the path. \cite{PAMPC} navigate the quadrotor between specified start and goal positions in an MPC framework, minimizing the speed and deviation of the projection of a particular point in the environment from the optical center of the on-board camera. \cite{MuraliACC} use differential flatness and optimize a combination of a co-visibility constraint which ensures the camera observes the change of bearing of a sufficient number of landmarks between consecutive keyframes, and the execution time of the trajectory.

In this paper, we consider the problem of minimizing the time required by a quadrotor to execute a given path, while maintaining a given set of landmarks within field of view of its on-board camera. Our motivation for studying this problem comes from an approximate, but efficient, motion planning strategy that involves splitting a trajectory planning task into a path planning phase in a lower dimensional space which ensures collision avoidance, followed by a time parametrization of that path which respects kinodynamic constraints of the quadrotor \cite{lavalle_2006}. The contribution of this paper is an algorithm for the time optimal path parametrization problem for a quadrotor with field of view constraints, along with a derivation of the convexity such constraints induce in a suitable parametrization of the problem. 
 
%

\section{Preliminaries and Problem Forumulation}
\label{sec:prelim}

\subsection{Coordinate Frame Conventions}

We use $\mathbf{p}^{\mathcal{A}} \in \mathbb{R}^3$ to denote the coordinates of an arbitrary point $\mathbf{p} \in \mathbb{R}^3$ with respect to frame $\mathcal{A}$, characterized by its origin $\mathcal{\mathbf{O}}_{\mathcal{A}} \in \mathbb{R}^3$ and orientation $R_{\mathcal{A}} \in SO(3)$. We define the transformation from frame $\mathcal{A}$ to frame $\mathcal{B}$ as an ordered pair $(\mathbf{t}_{\mathcal{B}}^{\mathcal{A}}, \ R_{\mathcal{B}}^{\mathcal{A}}) \in \mathbb{R}^3 \times SO(3)$, satisfying $\mathbf{t}_{\mathcal{B}}^{\mathcal{A}} = (R^{\mathcal{A}})^{-1} (\mathbf{O}_{\mathcal{B}} - \mathbf{O}_{\mathcal{A}})$ and $R_{\mathcal{B}}^{\mathcal{A}} = (R^{\mathcal{A}})^{-1}R_{\mathcal{B}}$. In other words, $\mathbf{t}_{\mathcal{B}}^{\mathcal{A}}$ and $R_{\mathcal{B}}^{\mathcal{A}}$ represent the displacement from the origin of $\mathcal{A}$ to the origin of $\mathcal{B}$ and the orientation of axes of $\mathcal{B}$, with respect to axes of $\mathcal{A}$, respectively. The relationship $\mathbf{p}^{\mathcal{A}} = \mathbf{t}_{\mathcal{B}}^{\mathcal{A}} + R_{\mathcal{B}}^{\mathcal{A}} \mathbf{p}^{\mathcal{B}}$ holds for an arbitrary point $\mathbf{p} \in \mathbb{R}^3$.

In this paper, we refer to three coordinate frames: $\mathcal{W}$, a fixed inertial world frame; $\mathcal{B}$, the frame attached to the body of the quadrotor (also refered to as the body frame); and finally, $\mathcal{C}$, the frame attached to the camera on board the quadrotor. 
We will assume the camera is rigidly attached to the quadrotor, and so the transformation $(\mathbf{t}_{\mathcal{C}}^{\mathcal{B}}, R_{\mathcal{C}}^{\mathcal{B}})$ will be fixed. As a result, $(\mathbf{t}_{\mathcal{C}}^{\mathcal{W}}, R_{\mathcal{C}}^{\mathcal{W}})$ will be a function of $(\mathbf{t}_{\mathcal{B}}^{\mathcal{W}}, R_{\mathcal{B}}^{\mathcal{W}})$.

\subsection{Dynamics Model} \label{sec:Dynamics}

The pose of the quadrotor is given by $(\textbf{x}_{\mathcal{B}}^{\mathcal{W}},R_{\mathcal{B}}^{\mathcal{W}}) \in \mathbb{R}^3 \times SO(3)$. 
Its state $\mathbf{q}$ additionally contains $\mathbf{v}$ and $\mathbf{\omega}$, the translational and angular velocity, respectively, and satisfies 
\begin{equation} \label{eq:DroneDynamics}
\underbrace{
\begin{bmatrix}
\dot{\mathbf{x}} \\
\dot{\mathbf{v}} \\
\dot{R} \\
\dot{\mathbf{\omega}} \\
\end{bmatrix}}_{\dot{\mathbf{q}}}
= \begin{bmatrix}
\mathbf{v} \\
\mathbf{g} \\ 
R [\mathbf{\omega}]_{\times} \\ 
-J^{-1} \ \mathbf{\omega} \times J \mathbf{\omega}
\end{bmatrix}
+
\begin{bmatrix}
0 & 0 \\
R \mathbf{e}_3  & 0 \\ 
0 & 0 \\ 
0 & J^{-1} \\
\end{bmatrix}
\underbrace{
\begin{bmatrix}
c \\
\mathbf{\tau}\\
\end{bmatrix}}_{\text{effective input}}
\end{equation}
where $J$ is the mass-normalized moment of inertia with respect to the body frame of the quadrotor. 
The input of the system consists of four mass-normalized motor thrusts $\mathbf{c} = [c_1 \ c_2 \ c_3 \ c_4]^T$, manifested through the resulting thrust $c \in \mathbb{R}$, and torque $\tau \in \mathbb{R}^3$, satisfying
\begin{equation} \label{eq:Thrust2ThrustTorque}
\begin{aligned}
\begin{bmatrix}
c \\
\mathbf{\tau} \\
\end{bmatrix}
= 
\underbrace{
\begin{bmatrix}
1 & 1 & 1 & 1 \\
-k_L & k_L & k_L & -k_L\\
-k_L & -k_L & k_L & k_L\\
-k_M & k_M & -k_M & k_M\\
\end{bmatrix}
}_{=: F}
\begin{bmatrix}
c_1 \\
c_2 \\
c_3 \\
c_4 \\
\end{bmatrix}
\end{aligned}
\end{equation}
for suitable choices of $k_L, \ k_M > 0$.  
Typically, motors can only supply a finite interval of positive thrusts which we model with constraints $c_{min} \leq c_i \leq c_{max}$ for all $1 \leq i \leq 4$. 
%

The differential flatness of the quadrotor \cite{MellingerKumar} means that the trajectory of its states and inputs can be uniquely recovered from algebraic functions of a finite number of derivatives of trajectories of its position $\mathbf{x}$ and heading $\mathbf{\psi}$:
\begin{equation} \label{eq:DifferentiallyFlatTrafo}
\small{
\begin{bmatrix}
\mathbf{q}(t) \\
\mathbf{c}(t) \\
\end{bmatrix}
 = \Phi \left( \left( \frac{d^n}{dt^n} \mathbf{x}(t) \right)_{n = 0,1,2,3,4} , \left( \frac{d^n}{dt^n} \mathbf{\psi}(t) \right)_{n = 0,1,2} \right)}
\end{equation}  

%



\subsection{Perception Model} \label{sec:perceptionModel}

We model the on-board camera using the pinhole projection principle. 
A point with coordinates $\mathbf{p}^{\mathcal{C}}$, with $\mathbf{p}_3^{\mathcal{C}} \geq 0$, is registered as a point $\left( \frac{\mathbf{p}_1^{\mathcal{C}}}{\mathbf{p}_3^{\mathcal{C}}}, \frac{\mathbf{p}_2^{\mathcal{C}}}{\mathbf{p}_3^{\mathcal{C}}} \right)$, provided it lies inside the region corresponding to the sensor array of the camera. 
Typically, the latter region is a rectangle, characterizing the horizontal and vertical extent of the field of view. 
In our case, we model the field of view as a symmetric circular cone consisting of points $\mathbf{p} \in \mathbb{R}^3$ satisfying 
\begin{equation*}
\mathbf{p}_3^{\mathcal{C}} \geq 0, \ \frac{|| [\mathbf{p}_1^{\mathcal{C}} \ \mathbf{p}_2^{\mathcal{C}}]^T ||_2}{\mathbf{p}_3^{\mathcal{C}}} \leq \tan \alpha,
\end{equation*}
for some fixed angle $\alpha \in (0, \frac{\pi}{2})$. Depending on the choice of $\alpha$, the latter relation can represent both conservative and relaxed characterizations of the field of view, in addition to inducing mathematical structure that can yield more efficient planning algorithms.

\subsection{Problem Formulation}
\label{sec:problem_formulation}

Assume we are given a smooth regular path $(\gamma, \psi) : [0, S_{end}] \rightarrow \mathbb{R}^3 \times S^1$, a set $S = \{l_1,...,l_n\}$ of stationary landmarks with known coordinates $\{\mathbf{l}_1^{\mathcal{W}}, ...., \mathbf{l}_n^{\mathcal{W}}\}$, and a map $\mathcal{M} : [0, S_{end}] \rightarrow 2^{S}$ specifying the desired subset of landmarks the on-board camera should have within field of view at every point along the path. 
Additionally, assume we are given a trajectory $\mathbf{n},\beta : [0, S_{end}] \rightarrow S^2 \times [0,\pi)$ of unit vectors and angles, denoting constraints that the body $z$ axis of the quadrotor should form an angle at most $\beta(s)$ with vector $\mathbf{n}(s)$  when its center of mass is located at position $\gamma(s)$. 
For example, the latter requirement arises in planning trajectories in which the quadrotor, modeled as an oblate spheroid, has to pass through narrow gaps \cite{falanga2017aggressive, Sayre-McCordKaraman}. 
Our problem consists of finding the shortest interval during which $\gamma$ can be traversed subject to aforementioned requirements as well as the constraint that the resulting trajectory be dynamically feasible.

The objective is to find the minimum $T > 0$ for which there exists a sufficiently smooth, strictly increasing map $s : [0, T] \rightarrow [0, S_{end}]$ so that trajectories of position and heading of the quadrotor defined by $\mathbf{x}(t) = \gamma(s(t))$ and $\mathbf{\psi}(t) = \psi(s(t))$ for all $t \in [0, T]$, respectively, induce a feasible trajectory. 
We define the set-valued map $\Lambda : \mathbb{R}^3 \times SO(3) \rightarrow 2^{\mathbb{R}^3}$ denoting the field of view of the camera on-board the quadrotor when the latter is in a given pose. 
We solve Problem $1$:
\begin{equation}
\begin{aligned}
\min & \ T \\ 
& s.t. \  \exists s : [0,T] \rightarrow [0, S_{end}] \\
& s(0) = 0, \ s(T) = S_{end} \\
& s \text{ is } \begin{cases}
\text{strictly increasing} \\
\text{twice differentiable}
\end{cases} \\
& \text{for all } t \in [0,T]:\\
& \hspace{5mm} \mathbf{x}(t) = \gamma(s(t)) \\
& \hspace{5mm} {
\begin{bmatrix}
\mathbf{q}(t) \\
\mathbf{c}(t) \\
\end{bmatrix}
 = \Phi \left( \left( \frac{d^n}{dt^n} \mathbf{x}(t) \right)_{n \leq 4} , \left( \frac{d^n}{dt^n} \mathbf{\psi}(t) \right)_{n \leq 2} \right)} \\
& \hspace{5mm} l_j \in \Lambda(t) \ \forall l_j \in \mathcal{M}(s(t)) \\
& \hspace{5mm} c_{min} \leq \textbf{c}(t) \leq c_{max}.
\end{aligned}
\end{equation} 

\section{Approach}
\label{sec:approach}

We assume the input to our algorithm is a regular, $C^2$ path of differentially flat outputs. 
Typically, such paths are represented as piecewise polynomial functions of an abstract real parameter \cite{MellingerKumar, RichterBryRoy}, although any sufficiently smooth path whose derivatives can be queried easily would suffice. 
Our approach consists of three stages. 
First, we will obtain an optimal time parametrization for a quadrotor with point mass dynamics. 
Then we will smooth the resulting parametrization in a suitable way in order to proceed with a solution for the original problem.

We will start by solving the time optimal path parametrization problem for a point mass model of a quadrotor with bounded total thrust. 
A convenient reparametrization of this problem involves rephrasing it in terms of a new decision variable, the square speed profile $h : [0, S_{end}] \rightarrow [0, +\infty)$ \cite{PfeifferJohani}. 
The link between $s$ from Section \ref{sec:problem_formulation} and $h$ is given by $h(s(t)) = \left( \frac{ds}{dt} \vert_{s^{-1}(t)} \right)^2$. 
The execution time, $\int_{0}^{S_{end}} \frac{ds}{\sqrt{h(s)}}$, is a convex function of $h$, and typical bounds on velocity and acceleration, given by 
\begin{equation*}
\mathbf{v}(s) = \gamma'(s) \sqrt{h(s)} \text{ and } \mathbf{a}(s) = \frac{1}{2}\gamma'(s)h'(s) + \gamma''(s) h(s), 
\end{equation*}
respectively, induce convex constraints on $h$ \cite{VerscheureDemeulenaere, LippBoyd}. For example, a bound on the maximum velocity is equivalent to a bound on $h$, whereas a bound on maximum total thrust amounts to 
\begin{equation*}
\left|\left| \frac{1}{2}\gamma'(s)h'(s) + \gamma''(s) h(s) - \mathbf{g} \right|\right|_2 \ \leq \ 4c_{max}, \ \forall s \in [0,S_{end}].
\end{equation*} 
Furthermore, in Section \ref{sec:convexity}, we will show that the requirement that landmark $l_j$ lies within field of view of the camera on board the quadrotor located at position $\gamma(s)$ and having heading vector $\mathbf{\psi}(s)$ reduces to the convex constraint
\begin{equation*}
(\mathbf{a}(s) - \mathbf{g}) \cdot (\mathbf{\delta} \times \mathbf{\psi^{\perp}}(s)) \ \geq \ \chi(s) \ \left|\left| \mathbf{\psi^{\perp}}(s) \times (\mathbf{a}(s) - \mathbf{g}) \right|\right|_2,
\end{equation*}
where $\mathbf{\psi^{\perp}}(s) := \mathbf{z_W} \times \mathbf{\psi}(s)$, and $\mathbf{\delta} \in \mathbb{R}^3$ and $\chi(s) > 0$ are suitable quantities depending only on $\mathbf{l_j}^W, \ \gamma(s),$ and $\mathbf{\psi}(s)$. Similarly, the constraint that the body $z$ axis of the quadrotor forms an angle no larger than $\beta$ with a unit vector $\mathbf{n}$ translates into 
\begin{equation*}
(\mathbf{a}(s) - \mathbf{g}) \cdot \mathbf{n} \  \geq \ \cos \beta \ || \mathbf{a}(s)- \mathbf{g} ||_2.
\end{equation*}

In summary, the relations above show that various constraints on kinematics and attitude of the quadrotor, at any particular point $s$ along the path, may be written in the form $\textbf{C}(s, h(s) ,h'(s)) \leq 0$ for an appropriate function $\textbf{C}$ which is component-wise jointly convex in its second and third variables. 
This motivates the formulation of the following problem \cite{VerscheureDemeulenaere, TOPPRA, HungarianManipulators, ConvexWaiter, LippBoyd}, which will be solved in the first as well as third stage of the approach, for suitable choices of $\textbf{C}$:

\begin{equation*} \label{ContinuousParameterProblem}
\begin{aligned}
& \underset{h : [a,b] \rightarrow [0,\infty)}{\text{minimize}}
& & \int_{a}^{b} \frac{ds}{\sqrt{h(s)}} \\
& \text{subject to}
& & \textbf{C}(s, h(s), h'(s)) \leq 0, \ \forall s \in [a, b), \\
&&& 0 \leq h(s) \leq B_u(s), \ s \in [a, b].
\end{aligned}
\end{equation*} 

Although the latter problem can be tackled by any generic convex optimization package, it possesses additional structure making it amenable to more efficient approaches \cite{HungarianManipulators, ConvexWaiter, TOPPRA}. 
One such algorithm \cite{TOPPRA} recovers values of the optimal profile at a sequence of points $0 = s_0 < \cdots < s_n = S_{end}$. 
The convexity of the problem implies that at an arbitrary discretization point $s_i$, the sets of square speeds that can be extended to feasible profiles on $[s_i, S_{end}]$ and $[0, s_i]$ are both intervals. 
The algorithm leverages this insight to recover an asymptotically optimal solution \cite{SpasojevicCDC} consisting of upper endpoints of such intervals, obtained incrementally in a pair of sweeps through $(s_i)_{i = 0}^n$ , as detailed in stage one of Algorithms \ref{alg:backward-forward} and \ref{alg:propogation-procedure}.

%
%

\begin{algorithm}[h]
\SetAlgoLined 
\KwData{$D = (s_i)_{i=0}^n$, $(B_l(s_i))_{i = 1}^n$, $(B_u(s_i))_{i = 1}^n$, $\psi^{\perp}$, $\gamma$, $\gamma'$,\\ \hspace{20mm} $\gamma'', \ (M(s_i))_{i = 0}^{n}, \ (n_i)_{i = 0}^{n}, \ (\beta_i)_{i = 0}^n,  \ c_{max}$} 

\KwResult{$(h_i)_{i=0}^n$}

\For{$i = 0$ to $n$}{
$l_i = B_l(s_i), \ h_i = B_u(s_i)$ \\
$R(s_i), \ R'(s_i), \ R''(s_i) \ \leftarrow \ null$ \\
\For{$j \in M(s_i)$}{
$v_j^{(i)} = l_j^{w} - \gamma(s_i)$ \\
$\chi_j^{(i)} = d \sin^2 \alpha + \cos \alpha \sqrt{|| v_j^{(i)} ||_2^2 - d^2 \sin^2 \alpha}$\\}
}

\For{stage = 1, 2}{

\For{$i = n-1$ to $0$}{

$l_{i}, \ h_{i} = Prop(stage, \ forward, \ \psi^{\perp}(s_i),$ \\
\hspace{20mm} $\gamma'(s_i), \ \gamma''(s_i), \ (v_j^{(i)})_{j = 1}^{| M(s_i) |},$ \\
\hspace{20mm} $(\chi_j^{(i)})_{j = 1}^{| M(s_i) |}, \ n_i, \ \beta_i, \ c_{max}, \ l_i, \ h_i,$ \\ 
\hspace{20mm} $ l_{i+1}, \ h_{i+1},\ R(s_i), \ R'(s_i), \ R''(s_i))$

\If{backward step infeasible}{return infeasible}
}

\For{$i = 0$ to $n-1$}{

$l_{i+1}, \ h_{i+1} = Prop(stage, \ forward, \ \psi^{\perp}(s_i),$ \\
\hspace{20mm} $\gamma'(s_i), \ \gamma''(s_i), \ (v_j^{(i)})_{j = 1}^{| M(s_i) |},$ \\
\hspace{20mm} $(\chi_j^{(i)})_{j = 1}^{| M(s_i) |}, \ n_i, \ \beta_i, \ c_{max}, \ l_i, \ h_i,$ \\ 
\hspace{20mm} $ l_{i+1}, \ h_{i+1},\ R(s_i), \ R'(s_i), \ R''(s_i))$

\If{forward step infeasible}{return infeasible}
}

\If{stage = 1}{
$ (R(s_i), \ R'(s_i), \ R''(s_i))_{i = 0}^n \ \leftarrow \ \ Smoothen((h_i)_{i = 0}^n,$\\ 
$(\gamma'(s_i), \ \gamma''(s_i))_{i = 0}^n, (\psi^{\perp}(s_i),\  \psi^{\perp'}(s_i), \ \psi^{\perp''}(s_i))_{i = 0}^n)$
}
}

\Return{$(h_i)_{i=0}^n$}
\caption{Backward-Forward Algorithm}
\label{alg:backward-forward}
\end{algorithm}

\begin{algorithm}
\SetAlgoLined
\KwData{$stage, \ type, \ \psi^{\perp}, \ \gamma', \ \gamma'', (v_j^{(i)})_{j = 1}^{| M|}, (\chi_j^{(i)})_{j = 1}^{|M|}, $ \\ $ n_i, \ \beta_i, \ c_{max}, \ l_i, \ h_i, \ l_{i+1}, \ h_{i+1}$}
\KwResult{$(l^*, h^*)$}

\eIf{type = backward}{
$t = [1 \ 0]^T$\\
}{
$t = [0 \ 1]^T$\\
}


\eIf{stage = 1}{
$\left\{\begin{array}{lr}
        l^* \\
        h^* \\
        \end{array}\right\} = \left\{\begin{array}{lr}
        \min \\
        \max \\
        \end{array}\right\} \  [h \ \tilde{h}]\ t^T$ \\
\vspace{2mm}
$ s.t. \ c = \frac{1}{2}\gamma' \ \frac{\tilde{h}-h}{\Delta s} + \gamma'' \ h - g$ \\
\vspace{1mm}
$ \hspace{6mm} c \cdot (v_j^{(i)} \times \psi^{\perp}) \geq  \chi_j^{(i)} || \psi^{\perp} \times c ||_2 \ \forall j \in M$ \\
\vspace{1mm}
$ \hspace{6mm} c \cdot n \geq \cos \beta_i \ || c ||_2$ \\
\vspace{1mm}
$ \hspace{6mm} || c ||_2 \leq c_{max}$ \\
\vspace{1mm}
$ \hspace{6mm} l_{i+1} \leq \tilde{h} \leq h_{i+1}, \ \ l_{i} \leq h \leq h_i $\\}
{
$\left\{\begin{array}{lr}
        l^* \\
        h^* \\
        \end{array}\right\} = \left\{\begin{array}{lr}
        \min \\
        \max \\
        \end{array}\right\} \  [h \ \tilde{h}]\ t^T$ \\
\vspace{2mm}
$ s.t. \ c = \frac{1}{2}\gamma' \ \frac{\tilde{h}-h}{\Delta s} + \gamma'' \ h - g$ \\
\vspace{1mm}
$\hspace{6mm} c_{\parallel} = c \cdot R(s_i)e_3$ \\
\vspace{1mm}
$\hspace{6mm} \tau =  J^{-1} \left( \ \Gamma(s_i) \times J \Gamma(s_i) + \Gamma'(s_i)\right) h $ \\ 
$ \hspace{20mm} + \ J^{-1} \left(\frac{1}{2} \Gamma''(s_i) \right) \frac{\tilde{h} - h}{\Delta s} $\\
\vspace{1mm}
$\hspace{6mm} F^{-1} \begin{bmatrix}
c_{\parallel} \\
\tau
\end{bmatrix} \geq c_{min}, \hspace{3mm}
F^{-1} \begin{bmatrix}
c \\
\tau
\end{bmatrix} \leq c_{max}$
\vspace{1mm}
$ \hspace{6mm} c \cdot (v_j^{(i)} \times \psi^{\perp}) \geq  \chi_j^{(i)} || \psi^{\perp} \times c ||_2 \ \forall j \in M$ \\
\vspace{1mm}
$ \hspace{6mm} c \cdot n \geq \cos \beta_i \ || c ||_2$ \\
\vspace{1mm}
$ \hspace{6mm} l_{i+1} \leq \tilde{h} \leq h_{i+1}, \ \ l_{i} \leq h \leq h_i $\\
}

\Return{$(l^*, h^*)$}
\caption{Propagation Procedure}
\label{alg:propogation-procedure}
\end{algorithm}

%

The output of the first stage of Algorithm \ref{alg:backward-forward} yields an optimal square speed profile assuming the thrust vector of the drone as input. 
Thus, the optimum will be a piecewise continuous function, and (as a result) so will the trajectory of planned body $z$ axes of the quadrotor. 
However, typical control algorithms assume the trajectory of orientations varies smoothly with time.


We overcome this challenge using the procedure \emph{Smoothen}, which convolves $\mathbf{z_B}(\cdot)$, the computed piecewise continuous trajectory of body $z$ axes (unit vectors directed along calculated inputs), with a Gaussian kernel of appropriate width. 
Concretely, the convolution of the trajectory $\mathbf{z_B} : [0,S_{end}] \rightarrow S^2$ with a Gaussian kernel of width $\sigma$,  $f_{\sigma}(s) = \frac{1}{\sqrt{2 \pi \sigma^2}} e^{-\frac{s^2}{2 \sigma^2}}$, is defined via  
\[
\mathbf{\tilde{z}_B}(s) = \int_{-\infty}^{+\infty} \mathbf{z_B}(u) f_{\sigma}(s-u) du.
\]
The latter function is $C^{\infty}$ with 
\[
\frac{d^n}{ds^n}\mathbf{\tilde{z}_B}(s) = \int_{-\infty}^{+\infty} \mathbf{z_B}(u) \frac{d^{n}}{ds^n}f_{\sigma}(s-u) du
\] 
for all $n \in \mathbb{N}$. 
We ensure a $C^{\infty}$ estimate of the trajectory of (normalized) $z$ axes by setting $\mathbf{\hat{z}_B}(s) = \frac{\mathbf{\tilde{z}_B}(s)}{|| \mathbf{\tilde{z}_B }(s)||_2}$ for all $s \in [0, S_{end}]$. 
The smooth estimate of the optimal trajectory of body $z$ axes, $\mathbf{\hat{z}_B}(s)$, together with specified heading vectors, $\mathbf{\psi}(s)$, uniquely defines a smooth trajectory of orientations $R(\cdot)$ as well as their higher order derivatives $R'(\cdot)$ and $R''(\cdot)$. 

At this stage, we face two hurdles:
\begin{enumerate*}[label=(\roman*)]
\item due to the underactuated nature of quadrotor dynamics, the trajectory $\mathbf{\hat{z}_B}(\cdot)$ of smoothened $z$ axes might not be realizable by any square speed profile, regardless of the range of thrusts the individual motors can exert; 
\item the algorithm has yet to account for the motor thrusts required to vary orientations of the quadrotor in the desired manner.
\end{enumerate*}


We begin by considering (i). Suppose $\mathbf{\overline{z}_B}(\cdot)$ is the trajectory of body $z$ axes of the quadrotor.
 The equation for acceleration implies 
\begin{equation} 
\mathbf{\overline{z}_B}(s) \times \left( \frac{1}{2}\gamma'(s) h'(s) + \gamma''(s) h(s) - \mathbf{g} \right) = \textbf{0}
\label{eq:nonHolonomyConstraints}
\end{equation}
which yields $\mathbf{\overline{z}_B}(s) \cdot (\gamma''(s) \times \gamma'(s)) \ h(s) = \mathbf{\overline{z}_B}(s) \cdot (\mathbf{g} \times \gamma'(s))$, as well as $\mathbf{\overline{z}_B}(s) \cdot (\gamma'(s) \times \gamma''(s)) \ h'(s) = 2 \ \mathbf{\overline{z}_B}(s) \cdot (\mathbf{g} \times \gamma''(s))$ for all $s \in [0, S_{end}]$. In particular, both the speed and acceleration of the quadrotor following a given path are fully determined by its body $z$ axis. 
Thus, for every $s \in [0, S_{end}]$, $\mathbf{\overline{z}_B}'(s)$ must be parallel to a vector depending on $\gamma'(s), \gamma''(s)$ and $\mathbf{\overline{z}_B}(s)$. 
Since this constraint is not necessarily satisfied by $\mathbf{\hat{z}_B}(\cdot)$, the second round of the backward-forward algorithm includes a relaxation of the constraint in Equation \ref{eq:nonHolonomyConstraints}, namely, 
\begin{equation*}
\left| \left|  \mathbf{\hat{z}_B}(s) \times \left( \frac{1}{2}\gamma'(s) h'(s) + \gamma''(s) h(s) - \mathbf{g} \right) \right| \right|_2 \leq \eta 
\end{equation*}
for some fixed (small) parameter $\eta > 0$.

%
%

(ii) stems from the fact that the first round of the backward-forward algorithm models the dynamics of the quadrotor as a point mass. 
However, the orientation of the quadrotor is important as it can only exert thrust along its body $z$ axis. 
   
We address the latter by assuming that the zeroth, first, and second derivatives of actual orientations are well approximated by quantities $R(s), \ R'(s),$ and $R''(s)$ computed by the smoothing procedure. 
We extend a derivation from \cite{RBMotions}, which allow us to formulate bounds on thrust of individual motors as convex constraints on $h$ and $h'$. Defining $\Gamma : [0, S_{end}] \rightarrow \mathbb{R}^3$ via 
$
[\Gamma(s)]_{\times} = R(s)^T R'(s)
$ for all $s \in [0,S_{end}]$, the equation for angular velocity yields
\begin{equation*}
\begin{aligned}
\begin{rcases}
\dot{R} = R(s) \ [\omega]_{\times} \\
\dot{R} = R'(s) \ \sqrt{h(s)}
\end{rcases} &  \Rightarrow 
[\omega]_{\times} = R(s)^T R'(s) \sqrt{h(s)} \\
& \Rightarrow \omega(s) = \Gamma(s) \sqrt{h(s)} \\
& \Rightarrow \dot{\omega}(s) = \frac{1}{2} \Gamma''(s) h'(s) + \Gamma'(s) h(s), 
\end{aligned}
\end{equation*} 
implying the total torque 
\begin{equation*}
\begin{aligned}
\tau(s) & = J^{-1} ( \omega \times J \omega + J \dot{\omega} ) \\
& =  \left(J^{-1} \ \Gamma(s) \times J \Gamma(s) + \Gamma'(s)\right) h(s) + \left(\frac{1}{2} J^{-1} \Gamma''(s) \right) h'(s) \\
\end{aligned}
\end{equation*} 
is a linear function of $h$ and $h'$. Finally, we recover the thrust 
\begin{equation*}
\begin{aligned}
c = (R(s)\mathbf{e}_3) \cdot \left(\frac{1}{2}\gamma'(s)h'(s) + \gamma''(s) h(s) - \mathbf{g} \right),
\end{aligned}
\end{equation*} 
an affine function of $h$ and $h'$. Putting this together, we get that the constraint that the motor thrusts lie in the range $[c_{min}, c_{max}]$ translates into the constraint $c_{min}\leq F^{-1}[c \ \tau]^T \leq c_{max}$, which is convex in $h$, meaning the last stage of the approach can be dealt with effectively by stage two of Algorithms \ref{alg:backward-forward} and \ref{alg:propogation-procedure}.


\section{Convexity}
\label{sec:convexity}

In this Section, we will show that attitude requirements arising from collision avoidance constraints and perception objectives translate into convex constraints on the square speed profile. 
We will assume the setup of a forward facing camera with an axially symmetric conical field of view as described in Section \ref{sec:perceptionModel}. 
Throughout this section we will use the relation for total thrust 
\begin{equation*}
c(s) = \frac{1}{2}\gamma'(s)h'(s) + \gamma''(s)h(s) - \mathbf{g}
\end{equation*}

\begin{note} \normalfont
Due to the closure of convex sets under arbitrary intersections, the following claim implies that the requirement that an arbitrary region of space belongs to the field of view of the camera on board the quadrotor at an arbitrary subset of points along the path, induces a convex constraint on the feasible square speed profile. 
This is particularly useful in scenarios where we require a \textit{set of different} landmarks to be kept within field of view at all times.   
\ignorespaces
\end{note}

\begin{claim}
Consider any fixed $s \in [0,S_{end}]$, unit heading vector $\mathbf{\psi}$ lying in the world $x-y$ plane, and stationary landmark with coordinates $\mathbf{l}^{\mathcal{W}}$. 
The requirement that the landmark lies within field of view of the camera on board the quadrotor with position $\gamma(s)$ and heading vector $\psi$ induces a convex constraint on the squared speed profile $h$.
\label{cl:fovConvexity}
\ignorespaces
\end{claim}

\begin{proof}
The assumption that the optical center of the on-board camera lies along the body $x$ axis of the quadrotor implies $\mathbf{p}_{\mathcal{C}}^{\mathcal{W}} = \mathbf{\gamma}(s) +  d \mathbf{x}_{\mathcal{B}}$ for a fixed $d > 0$.
The requirement that $\mathbf{l}^{\mathcal{W}}$ lies within field of view amounts to
\begin{equation*} 
\begin{aligned}
\alpha \geq \angle (\mathbf{l}^{\mathcal{W}} - \mathbf{p}_{\mathcal{C}}^{\mathcal{W}},\  \mathbf{x}_{\mathcal{B}}) \ = \  \angle (\mathbf{l}^{\mathcal{W}} - & \mathbf{\gamma}(s) - d\mathbf{x}_{\mathcal{B}}, \ \mathbf{x}_{\mathcal{B}}) \\
& \Leftrightarrow \\
(\mathbf{l}^{\mathcal{W}} - \mathbf{\gamma}(s) - d\mathbf{x}_{\mathcal{B}})^T \mathbf{x}_{\mathcal{B}}  \geq   \cos \alpha  \ ||  \mathbf{l}^{\mathcal{W}} - & \mathbf{\gamma}(s) - d\mathbf{x}_{\mathcal{B}} ||_2 \ || \mathbf{x}_{\mathcal{B}} ||_2 \\
&  \Leftrightarrow \\
\frac{1}{d} ( \mathbf{l}^{\mathcal{W}} - \gamma(s) - d \mathbf{x}_{\mathcal{B}})^T (d \mathbf{x}_{\mathcal{B}}) \ \geq \ \cos \alpha & \ || \mathbf{l}^{\mathcal{W}} - \gamma(s) - d \mathbf{x}_{\mathcal{B}} ||_2  \\
& \Leftrightarrow\\
|| \mathbf{l}^{\mathcal{W}} - \gamma(s) ||_2^2  - || \mathbf{l}^{\mathcal{W}} - \gamma(s) - d \mathbf{x}_{\mathcal{B}} & ||_2^2 - d^2 || \mathbf{x}_{\mathcal{B}} ||_2^2 \ \geq \\ 
2  d \ \cos \alpha \ || \mathbf{l}^{\mathcal{W}} -  \gamma(s) - & d \mathbf{x}_{\mathcal{B}} ||_2  \\
& \Leftrightarrow \\
|| \mathbf{l}^{\mathcal{W}} - \gamma(s) ||_2^2 - d^2 \sin^2 \alpha  & \geq  \\
 (|| \mathbf{l}^{\mathcal{W}}  - \gamma(s)- d& \mathbf{x}_{\mathcal{B}} ||_2 + d \cos \alpha)^2  \\ 
& \Leftrightarrow \\
\sqrt{|| \mathbf{l}^{\mathcal{W}} - \gamma(s) ||_2^2 - d^2 \sin^2 \alpha} - d \cos \alpha &  \geq \\
|| \mathbf{l}^{\mathcal{W}} - & \gamma(s) - d \mathbf{x}_{\mathcal{B}}||_2 \\
& \Leftrightarrow \\
( \mathbf{l}^{\mathcal{W}} - \gamma(s))^T & \mathbf{x}_{\mathcal{B}} \geq \\
d \sin^2 \alpha + \cos \alpha  (|| \mathbf{l}^{\mathcal{W}} - & \gamma(s) ||_2^2 - d^2 \sin^2 \alpha)^{\frac{1}{2}} \\
&  \Leftrightarrow \\
c(s)^T ( (\mathbf{l}^{\mathcal{W}} - \gamma(s)) \times \psi^{\perp}(s) ) \ \geq \  \chi(s) &|| \psi^{\perp}(s) \times c(s) ||_2,
\end{aligned}
\end{equation*}
where $\chi(s) := d \sin^2 \alpha + \cos \alpha \sqrt{|| \mathbf{l^W} - \gamma(s) ||_2^2 - d^2 \sin^2 \alpha}$. Since $c$ is an affine function of $h$, the latter inequality induces a convex constraint on $c$, and thus on $h$. 
\ignorespaces
\end{proof}

\begin{remark} \normalfont
Note that by redefining the heading, convexity continues to hold for an arbitrary on-board camera whose optical axis lies in the conic span of the body $x$ and $y$ axes of the quadrotor. 
Additionally, the convexity is preserved with a camera aligned with the body $z$ axis of the quadrotor.
This can be shown by a simple adaptation of the following claim.
\ignorespaces
\end{remark}

\begin{claim} 
Consider any fixed $s \in [0,S_{end}]$, unit vector $\mathbf{n} \in \mathbb{R}^3$, and angle $\beta \in [0, \frac{\pi}{2})$. 
The requirement that the angle between $\mathbf{n}$ and the body $z$ axis of the quadrotor, located at $\gamma(s)$, is no larger than $\beta$ induces a convex constraint on the square speed profile $h$.
\label{cl:GapFitting}
\ignorespaces
\end{claim}

\begin{proof} The constraint translates into
\begin{equation*}
\begin{aligned}
\beta \ \geq \ \angle ( \mathbf{n}, \mathbf{z}_{\mathcal{B}} ) \ & \Leftrightarrow \
\mathbf{n}^T \mathbf{z}_{\mathcal{B}} \ \geq \ \cos \beta \ || \mathbf{n} ||_2 \ || \mathbf{z}_{\mathcal{B}} ||_2 \\
& \Leftrightarrow \\
\mathbf{n}^T c(s) & \ \geq \ \cos \beta \  || c(s) ||_2,
\end{aligned}
\end{equation*}
which, as in the proof of Claim \ref{cl:fovConvexity}, induces a convex constraint on $c$, and thus on $h$.
\ignorespaces
\end{proof}

\begin{figure*}[!h]
    \centering
       \begin{subfigure}[t]{0.3\textwidth}
        \centering
        \includegraphics[trim=2.3cm 4cm 4cm 4cm, clip,width=6cm]{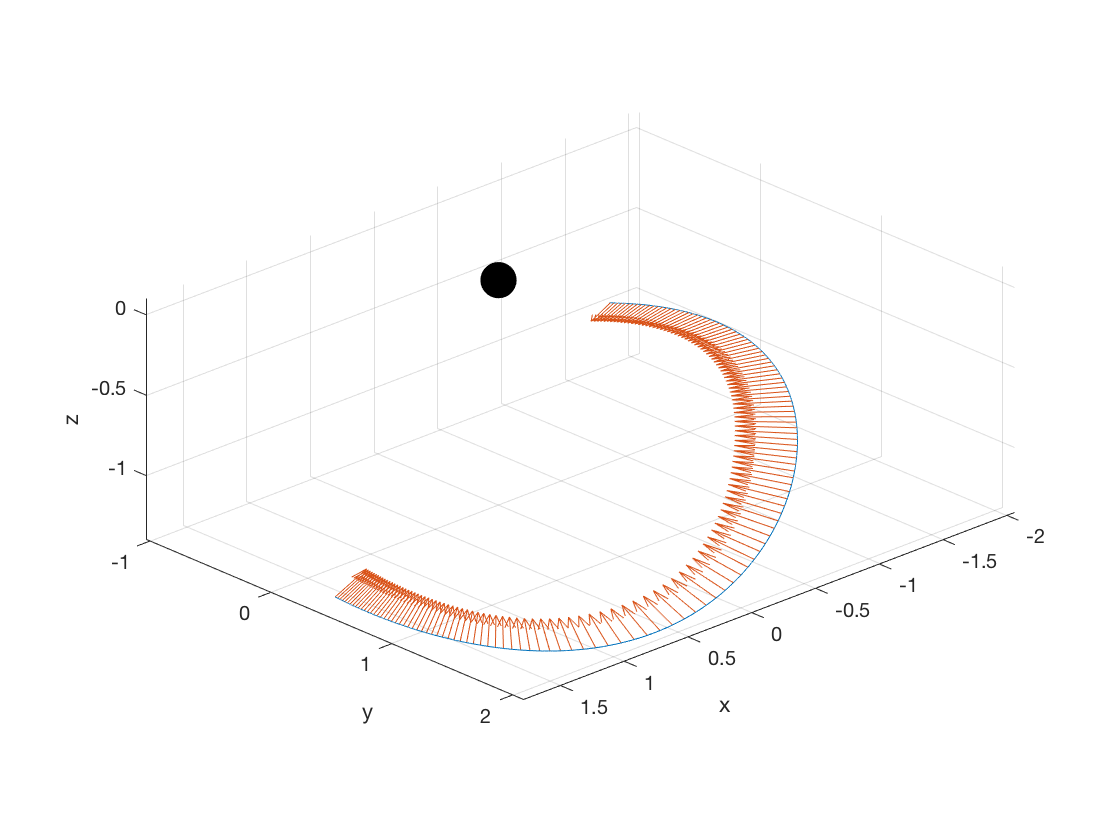}
        \caption{A sample trajectory generated by this method. The red arrows shows the attitude of the camera and the sphere represents the target to keep in view.}
        \label{fig:sim_explanation}
       \end{subfigure}
       ~
    \begin{subfigure}[t]{0.3\textwidth}
        \centering
        \includegraphics[trim=2.3cm 3.1cm 5cm 1.cm, clip, width=6cm]{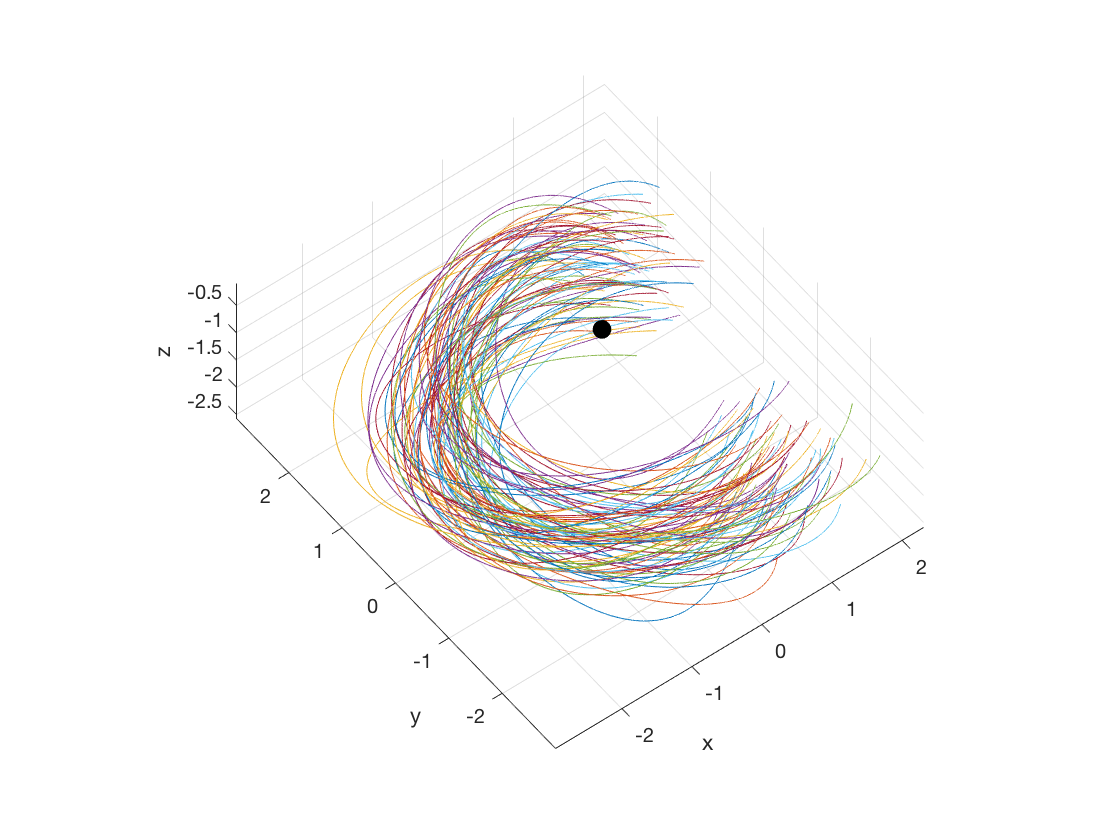}
        \caption{Trajectories generated using this approach. The sphere represents the target that is required to be kept in field of view.}
        \label{fig:squared_speed_prof}
           \end{subfigure}
      ~ 
     \begin{subfigure}[t]{0.3\textwidth}
        \centering
        \includegraphics[trim=2.3cm 1.cm 3.6cm 2cm, clip, width=4cm, height = 3.4cm]{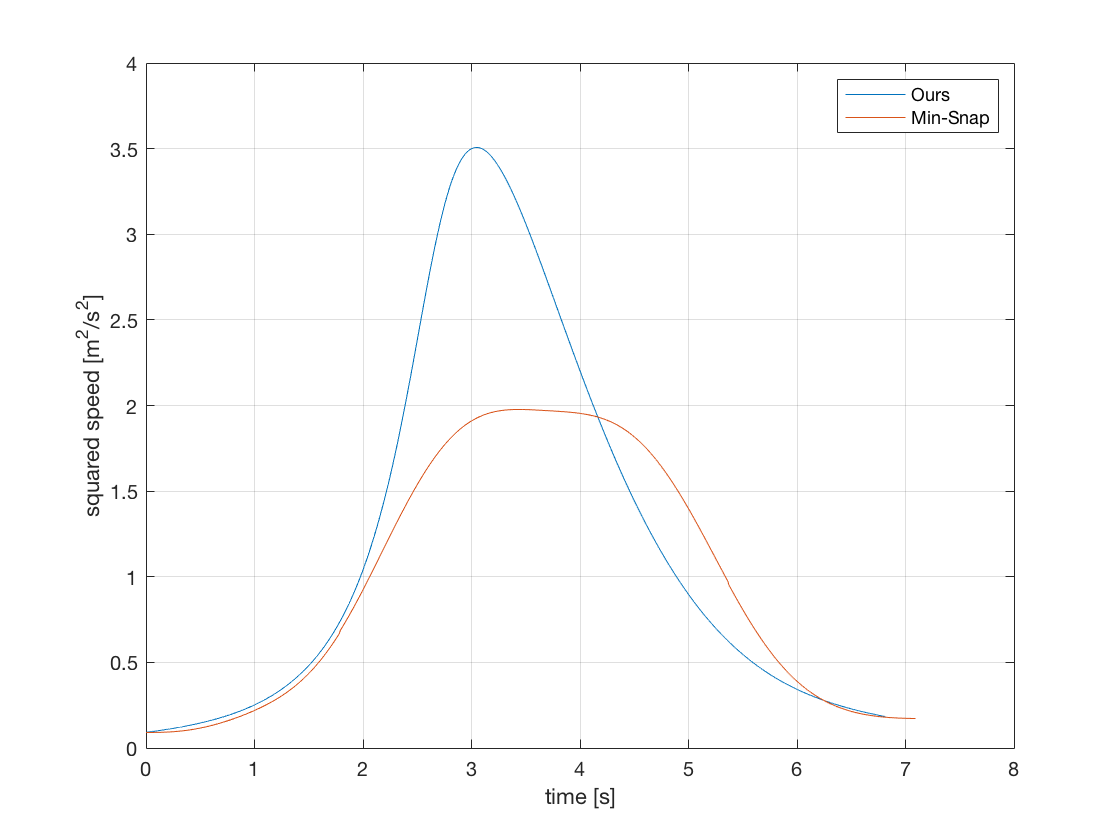}
         \caption{A comparison of the squared speed profile generated by a minimum snap optimization to our proposed method.}
          \label{fig:trajectories}
    \end{subfigure}%
    \caption{Experiments in simulation to validate the proposed approach.}
\end{figure*}

\begin{figure*}
\centering
        \begin{subfigure}[t]{0.25\textwidth}
        \centering
        \includegraphics[width=4cm]{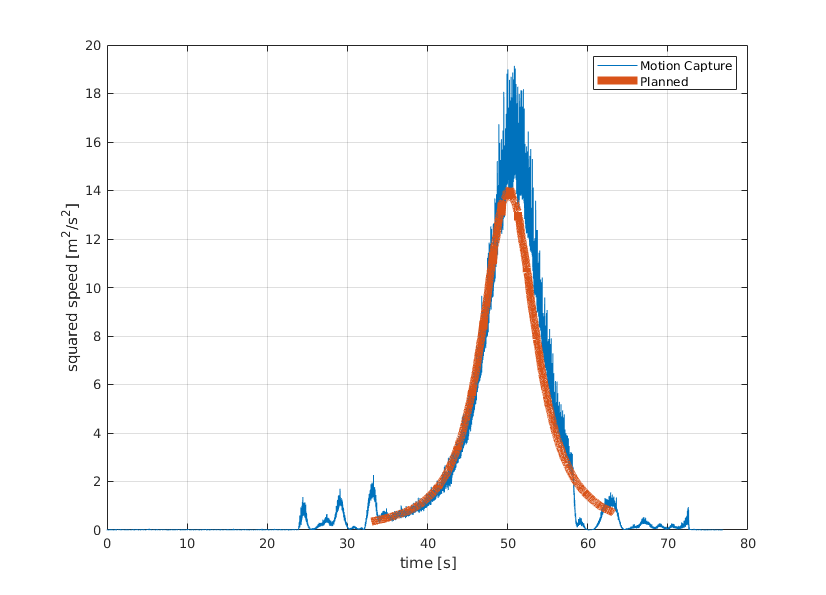}
         \caption{A comparison of the square speed profile obtained from flight data to the profile computed by the proposed algorithm.}
          \label{fig:real_experiment_sp_profile}
    \end{subfigure}
        ~
         \begin{subfigure}[t]{0.25\textwidth}
        \centering
        \includegraphics[width=4cm]{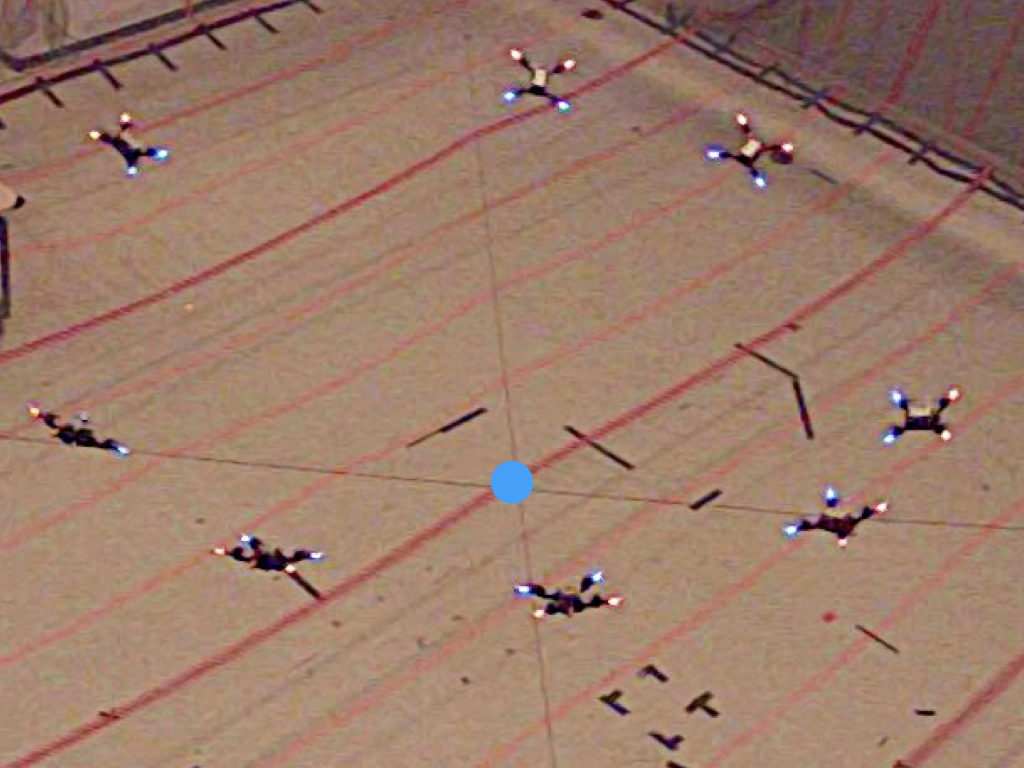}
         \caption{Spiral trajectory (clockwise) followed by the quadrotor is shown. The blue sphere represents the desired target to be kept in field of view.}
          \label{fig:real_experiment}
    \end{subfigure}%
          ~
         \begin{subfigure}[t]{0.25\textwidth}
        \centering
        \includegraphics[width=4cm]{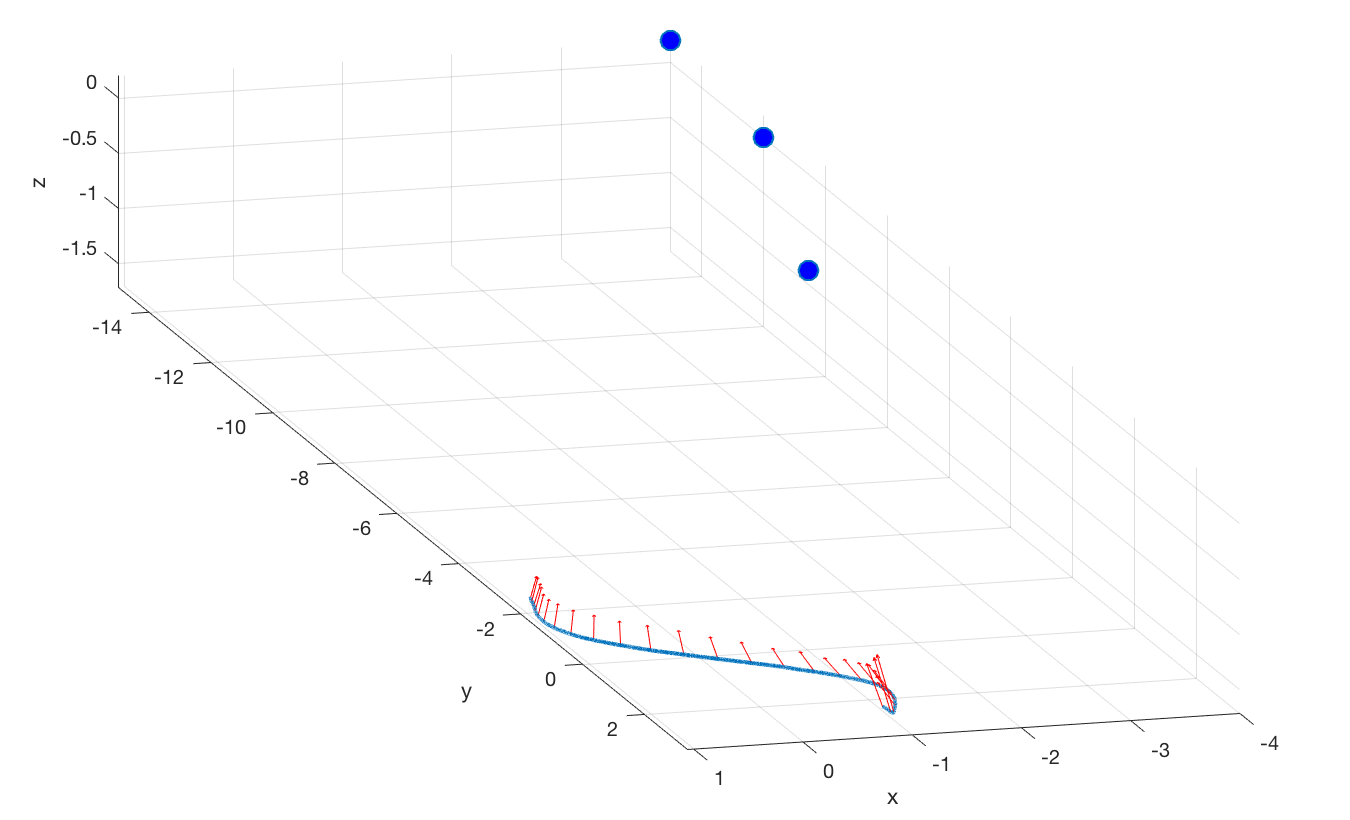}
         \caption{Slalom trajectory followed by the quadrotor. The blue sphere represent the desired targets to keep in view.}
          \label{fig:real_experiment_2}
    \end{subfigure}%
        \caption{Experiments on a real quadrotor to validate the proposed approach.} 
\end{figure*}

\section{Experiments}
\label{sec:experiment}


We first present two experiments in simulation.
The first experiment verifies that trajectories generated using our approach satisfy the field of view constraint.
This can be seen in Fig. \ref{fig:sim_explanation} where the red arrows show the attitude of the camera and the sphere represents the target.
In this experiment, we also compare the generated squared speed profile from our approach to one generated using a traditional minimum snap approach.
For fair comparison, we make a small modification (the search for the final time includes the field of view as a stopping condition) to the algorithm proposed in \cite{RichterBryRoy}.
The comparison of this can be seen in Fig. \ref{fig:squared_speed_prof} which shows that the trajectory generated using the proposed approach is more aggressive.
To validate the system, we first simulate the evolution of the state using the physics simulation presented in \cite{guerra2019flightgoggles}. 
We randomly choose a (fixed) set of waypoints, over 100 simulations, for the quadrotor to follow and a target that it has to keep in view.
Of the 100 random experiments performed using the Algorithm \ref{alg:backward-forward} and \ref{alg:propogation-procedure}, a total of 87 were deemed feasible  (shown in Fig. \ref{fig:trajectories}).
We use a state of the art controller presented in \cite{TalKaraman} to track the trajectories generated by the algorithm presented in Section \ref{sec:approach}.
The simulation is considered successful if the quadrotor is able to keep the target in view and track the reference trajectory with a reasonable accuracy (10 cm in this experiment).
In this experiment, the controller successfully tracks 84 trajectories.
We illustrate the validation of our algorithm on a custom built quadrotor platform by examining the performance of the quadrotor at a task involving  tracking a spiral trajectory while maintaining a desired point within field of view of its on-board camera. 
We use the trajectory tracking controller described in \cite{TalKaraman}, and motion capture system to obtain ground truth measurements.  
Fig. \ref{fig:real_experiment_sp_profile} effectively shows that the planned squared speed profile lies within dynamic capabilities of the quadrotor, whereas Fig \ref{fig:real_experiment} and Fig \ref{fig:real_experiment2} show the quadrotor maintain a point  and multiple points within field of view throughout the duration of the trajectory respectively.

\section{Conclusion}

This paper considered the time optimal path parametrization problem for quadrotors with field of view constraints. We proposed an algorithm based on the insight that field of view requirements translate into convex constraints on the square speed profile. Future work will additionally involve optimization over headings, as well as consideration of alternate methods of recovering a smooth estimate of the trajectory of optimal orientations from the output of stage one of our approach.   

\clearpage

\bibliographystyle{IEEEtran}
\bibliography{IEEEabrv,./references.bib}

\begin{thebibliography}{10}
\providecommand{\url}[1]{#1}
\csname url@samestyle\endcsname
\providecommand{\newblock}{\relax}
\providecommand{\bibinfo}[2]{#2}
\providecommand{\BIBentrySTDinterwordspacing}{\spaceskip=0pt\relax}
\providecommand{\BIBentryALTinterwordstretchfactor}{4}
\providecommand{\BIBentryALTinterwordspacing}{\spaceskip=\fontdimen2\font plus
\BIBentryALTinterwordstretchfactor\fontdimen3\font minus
  \fontdimen4\font\relax}
\providecommand{\BIBforeignlanguage}[2]{{%
\expandafter\ifx\csname l@#1\endcsname\relax
\typeout{** WARNING: IEEEtran.bst: No hyphenation pattern has been}%
\typeout{** loaded for the language `#1'. Using the pattern for}%
\typeout{** the default language instead.}%
\else
\language=\csname l@#1\endcsname
\fi
#2}}
\providecommand{\BIBdecl}{\relax}
\BIBdecl

\bibitem{AlphaPilot}
AlphaPilot, ``{AlphaPilot – Lockheed Martin AI Drone Racing Innovation
  Challenge},'' \url{https://www.herox.com/alphapilot}, 2019, [Online; accessed
  6-July-2019].

\bibitem{MellingerKumar}
D.~Mellinger and V.~Kumar, ``Minimum snap trajectory generation and control for
  quadrotors,'' 06 2011, pp. 2520 -- 2525.

\bibitem{Fliess}
M.~Fliess, J.~L{\'e}vine, P.~Martin, and P.~Rouchon, ``Flatness and defect of
  nonlinear systems: Introductory theory and examples,'' \emph{International
  Journal of Control}, vol.~61, pp. 13--27, 06 1995.

\bibitem{RichterBryRoy}
C.~Richter, A.~Bry, and N.~Roy, ``Polynomial trajectory planning for aggressive
  quadrotor flight in dense indoor environments,'' in \emph{ISRR}, 2013.

\bibitem{HehnDAndrea}
M.~Hehn and R.~D'Andrea, ``Real-time trajectory generation for quadrocopters,''
  \emph{IEEE Transactions on Robotics}, vol.~31, pp. 877--892, 2015.

\bibitem{LiuAtanasovKumar}
S.~Liu, K.~Mohta, N.~Atanasov, and V.~Kumar, ``Search-based motion planning for
  aggressive flight in se(3),'' \emph{IEEE Robotics and Automation Letters},
  vol.~PP, 10 2017.

\bibitem{PeninChaumette}
B.~Penin, R.~Spica, P.~Giordano, and F.~Chaumette, ``Vision-based minimum-time
  trajectory generation for a quadrotor uav,'' 09 2017, pp. 6199--6206.

\bibitem{SheckellsKobilarov}
M.~Sheckells, G.~Garimella, and M.~Kobilarov, ``Optimal visual servoing for
  differentially flat underactuated systems,'' 10 2016, pp. 5541--5548.

\bibitem{PAMPC}
D.~Falanga, P.~Foehn, P.~Lu, and D.~Scaramuzza, ``Pampc: Perception-aware model
  predictive control for quadrotors,'' 04 2018.

\bibitem{MuraliACC}
V.~{Murali}, I.~{Spasojevic}, W.~{Guerra}, and S.~{Karaman}, ``Perception-aware
  trajectory generation for aggressive quadrotor flight using differential
  flatness,'' in \emph{2019 American Control Conference (ACC)}, July 2019, pp.
  3936--3943.

\bibitem{lavalle_2006}
S.~M. LaValle, \emph{Planning Algorithms}.\hskip 1em plus 0.5em minus
  0.4em\relax Cambridge University Press, 2006.

\bibitem{falanga2017aggressive}
D.~Falanga, E.~Mueggler, M.~Faessler, and D.~Scaramuzza, ``Aggressive quadrotor
  flight through narrow gaps with onboard sensing and computing using active
  vision,'' in \emph{2017 IEEE international conference on robotics and
  automation (ICRA)}.\hskip 1em plus 0.5em minus 0.4em\relax IEEE, 2017, pp.
  5774--5781.

\bibitem{Sayre-McCordKaraman}
T.~{Sayre-McCord}, W.~{Guerra}, A.~{Antonini}, J.~{Arneberg}, A.~{Brown},
  G.~{Cavalheiro}, Y.~{Fang}, A.~{Gorodetsky}, D.~{McCoy}, S.~{Quilter},
  F.~{Riether}, E.~{Tal}, Y.~{Terzioglu}, L.~{Carlone}, and S.~{Karaman},
  ``Visual-inertial navigation algorithm development using photorealistic
  camera simulation in the loop,'' in \emph{2018 IEEE International Conference
  on Robotics and Automation (ICRA)}, May 2018, pp. 2566--2573.

\bibitem{PfeifferJohani}
F.~Pfeiffer and R.~Johanni, ``A {C}oncept for {M}anipulator {T}rajectory
  {P}lanning,'' \emph{IEEE Journal of Robotics and Automation}, vol. RA-3, pp.
  115 -- 123, 05 1987.

\bibitem{VerscheureDemeulenaere}
D.~Verscheure, B.~Demeulenaere, J.~Swevers, J.~De~Schutter, and M.~Diehl,
  ``Time-{O}ptimal {P}ath {T}racking for {R}obots: {A} {C}onvex {O}ptimization
  {A}pproach,'' \emph{IEEE Transactions on Automatic Control}, vol.~54, pp.
  2318 -- 2327, 11 2009.

\bibitem{LippBoyd}
T.~Lipp and S.~Boyd, ``Minimum-time speed optimisation over a fixed path,''
  \emph{International Journal of Control}, vol.~87, 02 2014.

\bibitem{TOPPRA}
H.~Pham and Q.~C. Pham, ``A {N}ew {A}pproach to {T}ime-{O}ptimal {P}ath
  {P}arameterization {B}ased on {R}eachability {A}nalysis,'' \emph{IEEE
  Transactions on Robotics}, vol.~34, pp. 645 -- 659, 06 2018.

\bibitem{HungarianManipulators}
L.~Consolini, M.~Locatelli, A.~Minari, A.~Nagy, and I.~Vajk, ``Optimal
  {T}ime-{C}omplexity {S}peed {P}lanning for {R}obot {M}anipulators,''
  \emph{IEEE Transactions on Robotics}, vol.~PP, 02 2018.

\bibitem{ConvexWaiter}
G.~Csorv\'asi, A.~Nagy, and I.~Vajk, ``Near {T}ime-{O}ptimal {P}ath {T}racking
  {M}ethod for {W}aiter {M}otion {P}roblem,'' vol.~50, 07 2017, pp. 4929--4934.

\bibitem{SpasojevicCDC}
I.~Spasojevic, V.~Murali, and S.~Karaman, ``Asymptotic optimality of a time
  optimal path parametrization algorithm,'' \emph{IEEE Control Systems
  Letters}, vol.~3, pp. 835--840, 2019.

\bibitem{RBMotions}
H.~Nguyen and Q.~C. Pham, ``Time-optimal path parameterization of rigid-body
  motions: Applications to spacecraft reorientation,'' \emph{Journal of
  Guidance, Control, and Dynamics}, vol.~39, pp. 1--5, 01 2016.

\bibitem{guerra2019flightgoggles}
W.~Guerra, E.~Tal, V.~Murali, G.~Ryou, and S.~Karaman, ``Flightgoggles:
  Photorealistic sensor simulation for perception-driven robotics using
  photogrammetry and virtual reality,'' \emph{arXiv preprint arXiv:1905.11377},
  2019.

\bibitem{TalKaraman}
E.~{Tal} and S.~{Karaman}, ``Accurate tracking of aggressive quadrotor
  trajectories using incremental nonlinear dynamic inversion and differential
  flatness,'' in \emph{2018 IEEE Conference on Decision and Control (CDC)}, Dec
  2018, pp. 4282--4288.

\end{thebibliography}

\end{document}